\documentclass[11pt]{amsart} 

\usepackage[utf8]{inputenc} 
\usepackage{amssymb}
\usepackage{url}
\usepackage{fixmath}

 \DeclareMathOperator{\atantwo}{atan2}

\newtheorem{lem}{Lemma}
\newtheorem{thm}{Theorem}

\usepackage{graphicx} 

\usepackage{array} 

\def\<{\langle}
\def\>{\rangle}

\title[Camera pose]{Characterization of the multiplicity of solutions for camera pose given two 
	vertically-aligned landmarks and accelerometer}
\author{Alexander R. Pruss}

\begin{document}
\begin{abstract}
We consider the problem of recovering the position and orientation of a camera equipped 
with an accelerometer from sensor images of two labeled landmarks whose positions in a coordinate
system aligned in a known way with gravity are known. This a variant on the much studied P$n$P 
problem of recovering camera position and orientation from $n$ points without any gravitational data.
It is proved that in three types of singular cases there are infinitely many solutions, in another type of 
case there is one, and in a final type of case there are two. A precise characterization of each type of case.
In particular, there is always a unique solution in the practically interesting case where the two landmarks 
are at the same altitude and the camera is at a different altitude. This case is studied by numerical simulation
and an implementation on a consumer cellphone. It is also proved that if the two landmarks are unlabeled,
then apart from the same singular cases, there are still always one or two solutions.
\end{abstract}

\maketitle
\section{Introduction}
Given $n$ landmark points with known coordinates in the object coordinate system, and a camera image of these points with 
the camera having known intrinsic parameters, the P$n$P problem is the problem of recovering the camera pose---location of optical center,
pitch, roll and yaw. It is known that with three landmarks, when the landmarks and the camera optical center are not
coplanar, there are at most four solutions (see \cite{WHZ} for a study of the multiplicity of solutions). With two distinct landmarks, 
given a camera pose that is a solution, rotation about the axis through the landmarks generates infinitely many other solutions,
as long as the camera is not on the axis.

But suppose we have two additional pieces of data: the camera is equipped with an accelerometer measuring gravity (or, equivalently,
has a coordinate system with a known orientation with respect to gravity), and the object coordinate system has the $z$-axis 
aligned vertically upwards, i.e., antiparallel to gravity~\cite{KBP10, DGMP13}. Some studies have considered a variant formulation
where the object comprising the landmarks is equipped with an accelerometer~\cite{DGMP14,DGMP21} instead of having its coordinate
system vertically aligned. But that formulation is mathematically equivalent to the one with the vertically aligned object coordinate 
system, since given the gravity vector from the accelerometer, it is easy to compute a rotation to a vertically-aligned coordinate system, 
and then we can compute the camera pose relative to this new coordinate system, and rotate to get a pose relative to the original coordinate 
system. We will thus work with the vertically-aligned object coordinate system for simplicity, and our characterization of the 
multiplicity of solutions will fully generalize.

It will be shown that there are three families of singular cases in which there are infinitely many solutions. In the remaining 
cases, there are either one or two solutions, and a complete characterization will be given of when the solution is unique. A particularly 
practically interesting subcase is where the line joining the landmarks is horizontal (i.e., perpendicular to gravity), but the 
plane through the landmarks and the camera's optical center is not horizontal (i.e., the camera is either below or above the landmarks).
In that subcase, there is always a unique solution. This could be used, for instance, to locate an accelerometer-equipped 
drone by using two landmarks on the ground, or to locate the precise pose of a game controller equipped with an accelerometer
and camera by means of landmarks below the expected level of the user's hand.

We assume that the gravitational force is parallel at all points. And as is usual in discussions of the P$n$P problem,
we will take the landmarks to be labeled or distinguishable in some way (e.g., identifiable by color, fiducial markers, or restrictions on 
camera position).

The present results expand on the proof in \cite{KBP10} that there are always at most two solutions, and correct 
correct the results of \cite{DGMP21} who found one of the three families of singular cases, and claimed
that in all the remaining cases there are two solutions (see also \cite{DGMP13,DGMP14}). Because of the difference between 
the present results and previous work on the problem, somewhat more detail will be included in the proofs than is usual.  

We can consider two kinds of cameras: cameras with a rectangular sensor and a pinhole mode as in previous studies, and spherical 
``360 degree'' cameras that capture a spherical image relative to a central optical center. The pinhole camera's optical center is the location of the pinhole while the spherical camera's optical center is the center of the image.
We can think of the pose of the camera as giving the transformation from camera coordinates to object coordinates, where
the camera coordinates have their origin at the optical center, and are defined by physical features of the camera (e.g., 
the orientation of the rectangular sensor or distinguished directions of the spherical image). 

Note that images from the pinhole camera can be projected to virtual images on a portion of the image sphere of the spherical 
camera and so without loss of generality the analysis can proceed entirely in terms of spherical cameras.

The above assumes the landmarks are labeled. If they are not, the number of solutions for camera pose  
must range between one and four by our main result in non-singular cases, but it is easy to show as a corollary
to our proofs for the labeled case that the number of solutions in non-singular cases must still be one or two, 
though we do not give a full characterization of when the number is one and when it is two.

Beyond the theoretical work, a numerical simulation and physical implementation on a consumer cellphone are presented 
in the special unique-solution configuration where the two labeled landmarks lie on a horizontal line.

\section{Sensor data and singular cases}\label{sec:singular}
Assume that the camera and object coordinate systems differ by a rigid motion.

With both camera types, the image data and intrinsic camera parameters can be used to straightforwardly compute the 
direction of normalized image vectors $\mathbold v_1$ and $\mathbold v_2$ 
from the the camera optical center to the two respective landmarks in the camera's coordinate system. 

\begin{figure}
\includegraphics[width=8cm]{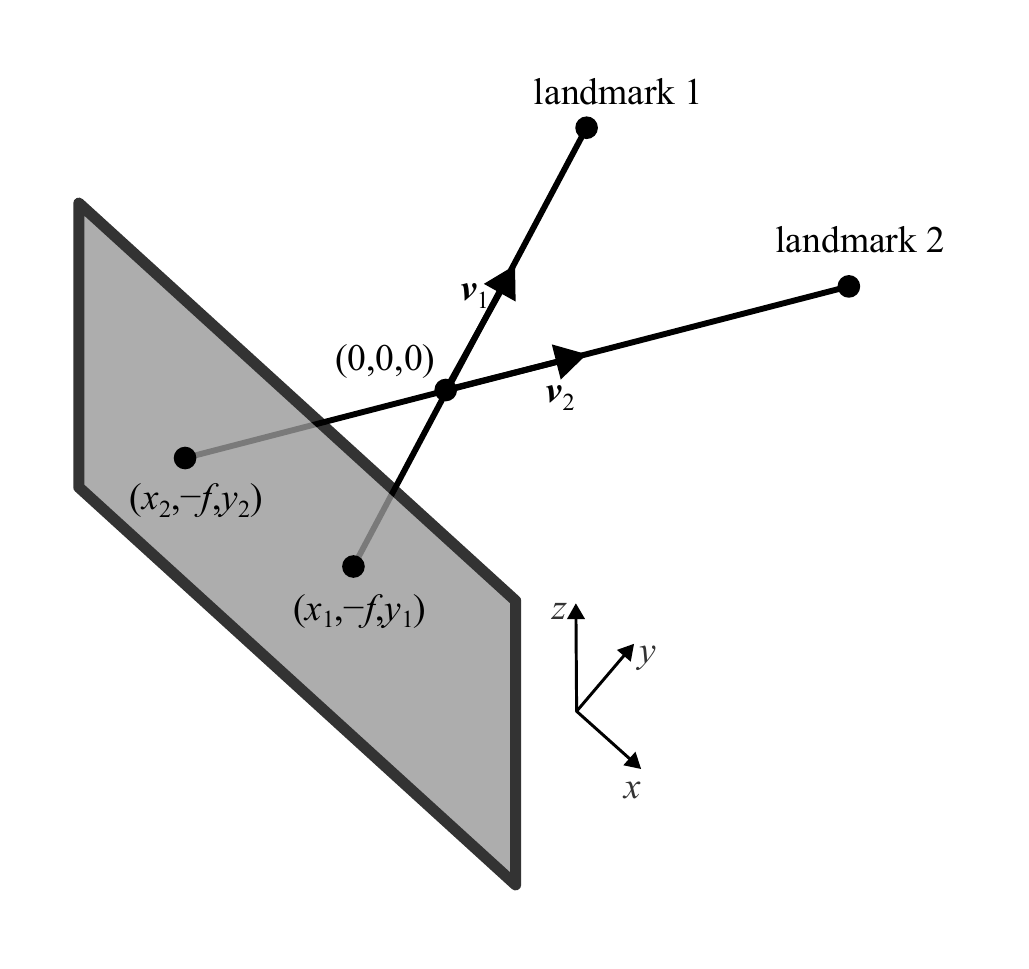}
\caption{A pinhole camera coordinate system with the points corresponding to $(x_i,y_i)$ coordinates on the sensor.}
\label{fig:planar}
\end{figure}
For instance, in the pinhole model, given a pinhole at the origin of the camera's 
coordinate system, we will have landmark images on the rectangular image sensor behind the pinhole located at 
three dimensional coordinates $\mathbold w_i$ (e.g., we can take $\mathbold w_i = (x_i,-f,y_i)$
where $f$ is the focal length, and $(x_i,y_i)$ are two-dimensional coordinates on the sensor surface with the optical axis
passing through $(0,0)$, as in Figure~\ref{fig:planar}). Then we can define $\mathbold v_i = -\mathbold w_i/|w_i|$. In the spherical model, supposing that 
the two landmarks occur at three-dimensional polar coordinates $(r,\theta_i,\phi_i)$ in the spherical image (with $r$ arbitrary), 
we can suppose that $\mathbold v_i=(\cos\phi_i\sin \theta_i, \sin\phi_i \sin \theta_i, \cos \theta_i)$. 

Additionally, we have the normalized vertical vector $\mathbold u$, antiparallel to the gravitational force, in the same 
coordinate system. We can now reduce the $\mathbold v_i$ and $\mathbold u$ data to three numbers relevant to determining 
the camera's position in the object coordinate system.

\begin{figure}
\includegraphics[height=6cm]{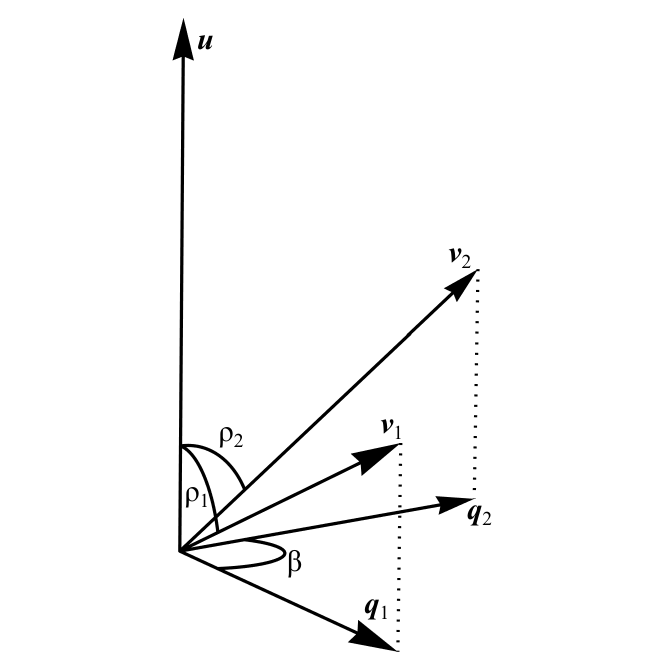}
\caption{The vertical tilt angles $\rho_i$ and the horizontal landmark angle $\beta$. The origin is the camera optical center.}
\label{fig:betarho}
\end{figure}
First, we have the \textit{image tilt angles}
$$
	\rho_i = \arccos(\mathbold u\cdot\mathbold v_i) 
$$	
between the landmark image vectors and vertical, ranging over $[0,\pi]$.
Second, we can project $\mathbold v_i$ on the horizontal plane (the plane perpendicular to gravity) to yield horizontal 
vectors
$$
	\mathbold q_i = \mathbold v_i - (\mathbold v_i\cdot\mathbold u) \mathbold u.
$$
	If neither $\mathbold q_i$ is zero (i.e., if
$\rho_i \in (0,\pi)$ for $i=1,2$), then let the \textit{horizontal landmark angle} $\beta$ be the angle in $(-\pi,\pi]$ from $q_1$ to $q_2$ measured counterclockwise, i.e.,
$$
\beta = \atantwo( (\mathbold q_1\times \mathbold q_2)\cdot \mathbold u, \mathbold q_1\cdot\mathbold q_2 ).
$$
See Figure~\ref{fig:betarho}. Otherwise, let $\beta=0$.

The information given by $\rho_1$, $\rho_2$ and $\beta$ is invariant under camera rotation about the vertical axis through
the optical center.
Furthermore, all the information relevant to determining camera position in object coordinates (or, equivalently, the 
translation component of the transformation $T$) is contained in $\rho_1$, $\rho_2$, $\beta$ and the respective positions 
of the two landmarks in object coordinates. For our full information for determining pose is given by $\mathbold v_1$, 
$\mathbold v_2$ and $\mathbold u$ in camera coordinates and the landmark positions in object coordinates. 
Now, given $\rho_1$, $\rho_2$ and $\beta$, suppose we have any normalized vectors $\mathbold v_i'$ such that 
$\rho_i = \arccos(\mathbold u\cdot \mathbold v'_i)$, and also such that $\beta$ is the signed angle from $\mathbold q_1'$ to 
$\mathbold q_2'$ (where $\mathbold q'_i$ is the horizontal projection of $\mathbold v_i$) if the $\mathbold q_i'$ are 
both non-zero (i.e., $\rho_i \in (0,\pi)$ for $i=1,2$). These vectors will be rotations of the ``correct'' normalized 
vectors $\mathbold v_i$ about $\mathbold u$, and will correspond to the camera having been rotated about the vertical 
axis through its optical center, which does not change the camera position.  Thus, $\rho_1$, $\rho_2$ and $\beta$ contain
all the information relevant to camera position.


Thus if there are infinitely many solutions for the camera position in object coordinates given $\rho_1$, $\rho_2$ and $\beta$ and 
landmark positions in object coordinates, it follows that we would still have infinitely many solutions given the 
fuller information provided by $\mathbold v_1$, $\mathbold v_2$ and $\mathbold u$. We can now describe the three families of 
singular cases.

For simplicity, it is always assumed that physical considerations preclude the camera optical center from coinciding 
with a landmark. If somehow that is possible (e.g., with a pinhole camera, one of the landmarks could technically lie 
in the pinhole), some of the discussion of singular cases will need slight modification.

One singular case with infinitely many solutions has been identified in \cite{DGMP21}, namely where the optical center and the landmarks are
colinear. In this case $\mathbold v_1 = \pm \mathbold v_2$ (only the $+$ option is available for the pinhole camera). Given 
a solution for camera position, translating the camera along a line joining the optical center and the landmarks without changing
the relative ordering of the camera and landmarks along that line does not change the values of $\rho_1$, $\rho_2$ and 
$\beta$ (or any other measured aspect of the sensor image). Therefore, there must be infinitely many solutions for camera position.

A second singular case, not identified in earlier work, is where the line between the two landmarks is vertical, and the 
camera optical center is not on this line. In such a case, rotating the camera about the line joining the landmarks 
does not change $\rho_1$, $\rho_2$ or $\beta$, and hence given any solution for camera position, there are infinitely many others. 

\begin{figure}
\includegraphics[width=6cm]{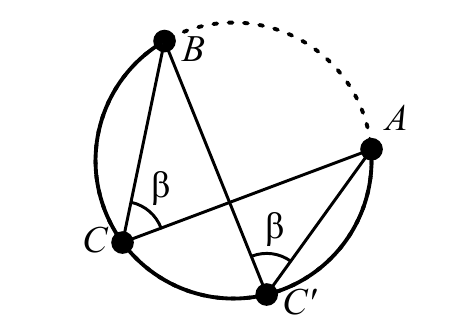}
\caption{Third singular case: The possible positions for the camera line on the interior of the solid segment $S$ between 
co-horizontal landmarks $A$ to $B$.}\label{fig:circle}
\end{figure}
The third singular case is perhaps the most interesting. Suppose that the two landmarks $A$  and $B$ and the camera optical 
center $C$ all lie on the same horizontal plane $K$, but the three points are not colinear. It is a standard fact of plane 
geometry that the locus of points $C'$ such that the signed angle $AC'B$ is equal to $\beta$ will be the interior of 
a segment $S$ of a circle passing through $A$ and $B$ where the circle's radius depends on $\beta$ and the distance $|AB|$
(see Figure~\ref{fig:circle}). 
If we move the camera to different points $C'$ on $S$, if necessary rotating it around the vertical axis to ensure that the 
landmarks remain in the sensor image in the case of the pinhole camera (which does not change the camera position), the value
of $\beta$ will be unchanged, and we will always have $\rho_1=\rho_2=\pi/2$. Thus, again, we have infinitely many solutions. This singular case will 
be of some importance when we discuss implementation in Section~\ref{sec:coplanarity}.

In the next section we will prove that these are the only singular cases, and that in the remaining configurations there 
are at most
two solutions for camera pose.

Note that it is easy to use the landmark positions as well as $\rho_1$, $\rho_2$ and $\beta$ to determine whether the camera is 
in a singular position, and if so, which one. The first singular case happens when $\rho_1$ and $\rho_2$ are both in 
$\{0,\pi\}$ (the line will be vertical) or when $\beta=0$ and $\rho_1=\rho_2$ (both landmarks produce a single camera image)
or when $\beta=\pi$ and $\rho_1=\pi-\rho_2$ (the landmarks lie on opposite sides of the optical center). The second singular
case happens when the landmarks are on a vertical line, $\beta=0$ and $\rho_1,\rho_2 \in (0,\pi)$. The third singular 
case happens when $\rho_1=\rho_2=\pi/2$. 

\section{Characterization of solutions}
Say that the slope from $(x_1,x_2,x_3)$ to $(y_1,y_2,y_3)$ is $(y_3-x_3)/((x_1-y_1)^2+(x_2-y_2)^2)^{1/2}$, allowing this 
to be $\pm \infty$. For brevity, we will refer to the position of the camera optical center as just the position of the camera.
Our main result is as follows.

\begin{thm}\label{th:solutions} 
	Assume that the $z$-axis of the object coordinates is antiparallel to gravity and that we are given  
	the upward vector $\mathbold u$ and the vectors $\mathbold v_i$ in camera coordinates from the 
	camera to the respective landmarks, as well the object coordinates of the respective landmarks.

	In each of the following three singular cases, there are infinitely many solutions for camera position: 
\begin{enumerate}
	\item[(i)] the two landmarks and camera 
	are colinear
	\item[(ii)] a vertical line passes through both landmarks
	\item[(iii)] a horizontal plane passes through both landmarks and the camera. 
\end{enumerate}
	
	If none of (i)--(iii) obtain, there are at most two solutions for camera pose, and furthermore there is exactly 
	one solution if and only if 
\begin{enumerate}
	\item[(a)] $\max(|s_1|,|s_2|) \ge |s|$ where $s_i$ be the slope from the camera to landmark $i$ and $s$ is the 
	slope from landmark $1$ to landmark $2$, or 
	\item[(b)] $|s|$ is positive, $\max(|s_1|,|s_2|)$ is finite, and the camera lies on the plane $L$ defined by both 
	landmarks and a horizontal line $N$ through landmark $1$ and perpendicular to the line through the landmarks.
\end{enumerate}
\end{thm}

In particular, by (a) there is always a unique solution if $s=0$, i.e., the two landmarks are at the same altitude. 
This can be useful for applications where one can arrange the landmarks at the same altitude, and ensure that 
the camera is always above or always below the landmarks, thereby avoiding all multiplicity 
of solutions.  

\begin{figure}
\includegraphics[width=8cm]{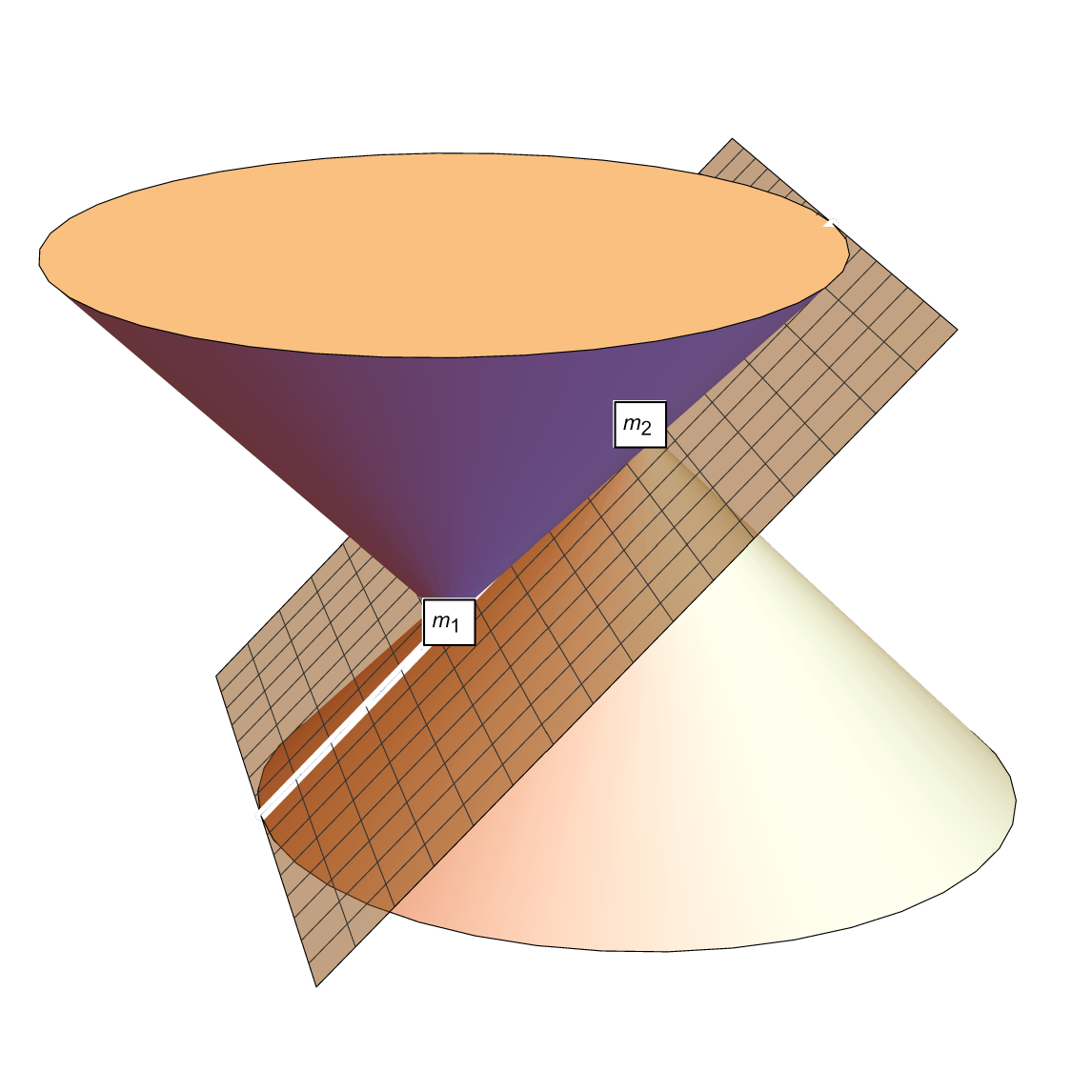}
\caption{Region of unique solution for landmarks at different altitudes.}
\label{fig:cones}
\end{figure}
If the landmarks are not at the same altitude nor on the same vertical line, the region of unique solution described by (a) consists 
of the union of two closed circular cones with vertical axes of symmetry and generator slope $|s|$, with apices at the landmarks, 
with the cone at the upper
landmark being pointed downward and the cone at the lower landmark being pointed upward (see Figure~\ref{fig:cones}), and with the line
through the landmarks removed (type (i) singularity). The region of unique solution described by (b) is a two-dimensional plane 
and hence unlikely to be of practical use.

The existence of infinitely many solutions in the singular cases was proved in the previous section.

Note that the following observation does not need an accelerometer.

\begin{lem}\label{lem:position-pose} If the two landmarks and camera are not colinear, then given the 
	camera position in object coordinates, and the normalized vectors $\mathbold v_i$ from the camera to 
	the respective landmarks, camera pose is uniquely determined.
\end{lem}

\begin{proof}[Proof of Lemma~\ref{lem:position-pose}]
Let $RT$ be the transformation from object to camera coordinates, where $R$ is a rotation through the origin
and $T$ is the translation which is known given the camera position. 
Let $\mathbold {v'}_i = T\mathbold m_i/|T\mathbold m_i|$ be unit vectors from the camera to the landmarks.
Then we must have $R\mathbold {v'}_i = \mathbold v_i$. Since $\mathbold v_1$, $\mathbold v_2$ and the origin are not 
colinear, $R$ is unique.
\end{proof}

Write $\mathbold m_i = (m_{i1},m_{i2},m_{i3})$ for the landmark positions in object coordinates. Let 
$$
	H=m_{23}-m_{13}
$$ 
be the altitude difference between the landmarks and let
$$
	d=((m_{11}-m_{21})+(m_{12}-m_{22})^2)^{1/2}
$$
be the distance between the two landmarks' projections onto a horizontal plane. Let $\mathbold c=(c_1,c_2,c_3)$ be the camera 
position in object coordinates. Let 
$$
	h_i=c_{3}-m_{i3}
$$
be the height of the camera above landmark $i$, and let
$$
	d_i=((m_{i1}-c_{i1})+(m_{i2}-c_{i2})^2)^{1/2}
$$
be the distance between the camera and landmark positions projected onto a horizontal plane.

\begin{figure}
\includegraphics[width=4cm]{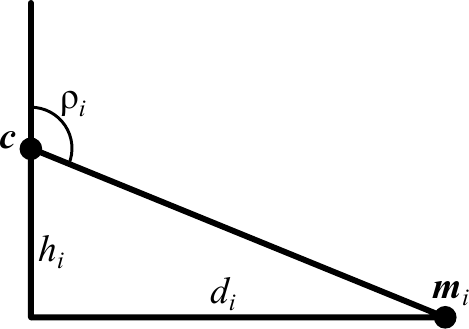}
\caption{The tilt angle $\rho_i$ from camera to landmark, the height $h_i$ of the camera above landmark, and the 
horizontal distance $d_i$ to the landmark.}
\label{fig:triangle}
\end{figure}
Assume now that we are not in any singular case.
Then:
\begin{equation}\label{eq:height}
	h_i = -d_i \cot \rho_i\text{ if }\rho_i \notin \{ 0,\pi \}.
\end{equation}
and 
\begin{equation}\label{eq:height-2}
	d_i = -h_i \tan \rho_i\text{ if }\rho_i \ne \pi/2
\end{equation}
(see Figure~\ref{fig:triangle}).

Additionally by our stipulation on $\beta$ in the case where the camera lies on a vertical line through a landmark:
\begin{equation}\label{eq:beta-stip}
	\beta = 0\text{ if }\rho_i \in \{0,\pi\}\text{ for some }i.
\end{equation}
By the law of cosines:
\begin{equation}\label{eq:width}
	d_1^2 + d_2^2 - 2d_1d_2 \cos \beta = d^2
\end{equation}
if the $d_i$ are both strictly positive. But if $d_i$ is zero for some $i$, the camera is on the same vertical 
line as landmark $i$, so $d=d_{3-i}$ and \eqref{eq:width} also holds (in this case $\beta=0$ was stipulated).
Finally:
\begin{equation}\label{eq:difference}
	h_1 - h_2 = H.
\end{equation}
These yield four equations in the four unknowns $h_1$, $h_2$, $d_1$ and $d_2$. 

\begin{lem}\label{lem:hhdd} 
	Suppose we are given two landmarks with positions $\mathbold m_i$ in object coordinates,
	as well as angles $\rho_1,\rho_2\in [0,\pi]$ and $\beta\in [-\pi,\pi]$ such that none 
	of the following cases obtain: 
	\begin{enumerate}
	\item[(i)] $\rho_1,\rho_2\in \{0,\pi\}$
	\item[(ii)] a vertical line passes through both landmarks
	\item[(iii)] $\rho_1=\rho_2=\pi/2$. 
	\end{enumerate}

	Suppose we 
	are given numbers 
	$h_1$, $h_2$, $d_1$ and $d_2$ satisfying \eqref{eq:height}--\eqref{eq:difference}, with 
	$d_1$ and $d_2$ non-negative.
	For $i=1,2$, suppose that 
\begin{equation}\label{eq:sign}
	\text{ if }d_i=0,\text{ then: } (h_i>0\text{ and }\rho_i=\pi)\text{ or }(h_i<0\text{ and }\rho_i=0).
\end{equation}
	
	Then there is a unique camera 
	position $\mathbold c$ in object 
	coordinates with image tilt angles $\rho_1$ and $\rho_2$, horizontal image angle $\beta$, and 
	the given parameters $h_1$, $h_2$, $d_1$ and $d_2$.
\end{lem}

The three conditions in Lemma~\ref{lem:hhdd} correspond to the three singular cases if $\rho_1$, $\rho_2$ and $\beta$ 
are derived from a camera pose. Furthermore, \eqref{eq:sign} is automatically satisfied if the $d_i$ and $h_i$ parameters 
come from image angles of a camera pose. 

\begin{proof}[Proof of Lemma~\ref{lem:hhdd}]
	Let $p$ be the projection $p(x_1,x_2,x_3)=(x_1,x_2)$ from three dimensional space to a horizontal plane $K$.
	
	Equation \eqref{eq:width} guarantees that $d_1$, $d_2$ and $d$ satisfy
	the triangle inequalities and hence are the sides of a possibly degenerate triangle. 
	
	Since $d>0$ as we are not in case (ii), it follows by the triangle inequalities that 
	at most one of the $d_i$ is zero. 
	
	Suppose first that $d_1=0$ and $d_2>0$. In this case, any 
	camera position $\mathbold c$ with parameters $h_1$, $h_2$, $d_1$ and $d_2$
	will satisfy $\mathbold c = (m_{11},m_{12},m_{13}+h_1)$, and hence there is at most one 
	camera position with the specified parameters $h_1$, $h_2$, $d_1$ and $d_2$. 
	
	We now show that 
	this $\mathbold c$ has all the required parameters and image angles. Clearly, it has the 
	correct $d_1$ and $h_1$ parameters. By \eqref{eq:difference} it also has the correct 
	$h_2$ parameter. The triangle inequalities imply that $d_2=d$, and $\mathbold c = (m_{11},m_{12},m_{13}+h_1)$
	indeed has $d_2$ parameter equal to the distance between the landmarks.
	
	We stipulated the horizontal image angle is zero if the camera is on the same vertical
	line as a landmark, and this is the case for our $\mathbold c$, while $\beta=0$ by \eqref{eq:beta-stip}
	since $\rho_1 \in \{0,\pi\}$ by \eqref{eq:sign}.
	The image tilt angle for $\mathbold c$ and landmark $1$ is equal to $\rho_1$ as $\rho_1$ satisfies
	\eqref{eq:sign}. It remains to check the image tilt angle for $\mathbold c$ and landmark $2$. 	
	Call this angle $\rho_2'$. Then $\rho_2' \in (0,\pi)$ as $d_2>0$ and $\cot\rho_2'=-h_2/d_2$.
	Now, $\rho_2$ cannot be in $\{0,\pi\}$ by \eqref{eq:height-2} as $d_2>0$, so $\cot\rho_2=\cot\rho_2'$
	by \eqref{eq:height}, and hence $\rho_2=\rho_2'$ as both angles are in $(0,\pi)$.
		
	The case where $d_2=0$ and $d_1>0$ follows by swapping indices. 

	Now suppose that $d_1$ and $d_2$ are both positive.
	Let $A=p\mathbold m_1$ and $B=p\mathbold m_2$. By the triangle inequalities, there are at most two points $C$ in $K$ 
	such that $|AC|=d_1$ and $|BC|=d_2$. For each of these points, the unsigned angle $ACB$ is the same. Call that
	angle $\beta'\in [0,\pi]$. By the law of cosines, we will have
$$
	d_1^2+d_2^2-2d_1d_2\cos \beta' = d^2.
$$
	Since $d_1>0$ and $d>0$, it follows from \eqref{eq:width} that $\cos\beta' = \cos\beta$. As 
	$\beta \in (-\pi,\pi]$, we must have $|\beta|=\beta'$. We now uniquely choose $C=(c_1,c_2)$ so that the signed angle 
	$ACB$ equals $\beta$. 
	
	Let $c_3=m_{13}+h_1$ and let $\mathbold c = (c_1,c_2,c_3)$. Then the image horizontal angle for a camera
	at $\mathbold c$ for the landmarks at $\mathbold m_i$ will be $\beta$ by the choice 
	of $(c_1,c_2)$. Moreover, the horizontal distance parameters $d_1$ and $d_2$ will be as given, and the
	heights $h_1$ and $h_2$ will also be as gvien.
	
	We now need to check the tilt angles. By \eqref{eq:difference}, we have $c_3=m_{i3}+h_i$ for $i=1,2$. 
	The tilt angle $\rho'_i$ for a camera at $\mathbold c$ with respect to landmark $i$ will 
	satisfy 
$$
	h_i = -d_i \cot \rho'_i.
$$
	By \eqref{eq:height} and \eqref{eq:height-2}, we have $h_i = -d_i \cot \rho_i$. Thus $\cot \rho'_i = \cot \rho_i$. 
	Since $\cot$ is one-to-one on $(0,\pi)$, it follows that $\rho'_i = \rho_i$. 
\end{proof}

%

\begin{proof}[Proof of Theorem~\ref{th:solutions}]
Assume we are not in any of the singular cases, as they were covered in Section~\ref{sec:singular}.

Consider first the special case where $\rho_i \in \{ 0, \pi \}$ for some $i$. 
Then the camera lies on the same 
vertical line as landmark $i$ so $d_i=0$. It cannot lie on the same vertical line as landmark $j=3-i$ or we would 
have a singular case. Therefore, $0<\rho_j<\pi$, and we can see (Figure~\ref{fig:triangle}) that $h_j = -d_j\cot\rho_j = -d\cot\rho_j$, since $d_j=d$
as $d_i=0$. As $h_i = h_j+m_{3j}-m_{3i}$, we have all four parameters $d_i$, $d_j$, $h_i$ and $h_j$ uniquely 
determined and the proof is completed by Lemma~\ref{lem:hhdd} once we note that if $\rho_i\in\{0,\pi\}$, then 
$|s_i|=\infty$. 

For the rest of the proof assume $0<\rho_i<\pi$ for $i=1,2$.
We cannot have $\rho_1=\rho_2=\pi/2$ as that is singular case (iii). Choose $i$ such that $\rho_i \ne \pi/2$, 
and let $j=3-i$. Let $\delta=h_i-h_j=m_{j3}-m_{i3}$. This is $H$ if $i=1$ and $-H$ if $i=2$ by \eqref{eq:difference}. 

We have $d_i = -(h_j+\delta) \tan\rho_i$ and $h_j = -d_j \cot\rho_j$ by \eqref{eq:height}--\eqref{eq:height-2}. Thus:
\begin{equation}\label{eq:di}
	d_i = (d_j\cot\rho_j-\delta)\tan\rho_i.
\end{equation}
Putting this into \eqref{eq:width} and collecting terms with the Mathematica computer algebra system we get the quadratic:
\begin{equation}\label{eq:quad}
	a d_j^2 + b d_j + c = 0
\end{equation}
where 
\begin{equation}\label{eq:a}
	a = 1-2\cos\beta\cot\rho_j\tan\rho_i + (\cot\rho_j\tan\rho_i)^2,
\end{equation}
\begin{equation}\label{eq:b}
	b = 2\delta\tan\rho_i (\cos\beta-\cot\rho_j\tan\rho_i)
\end{equation}
and
\begin{equation}\label{eq:c}
	c = \delta^2\tan^2\rho_i - d^2.
\end{equation}

Note that $a>0$ unless perhaps $\cos\beta = 1$. Suppose $\cos\beta=1$, so $\beta=0$. Then $a=(1-\cot\rho_j\tan\rho_i)^2$.
This can be zero only if $\rho_j=\rho_i$. But the only way to get $\beta=0$, $d_1>0$ and $d_2>0$ and $\rho_i=\rho_j$ is 
if the camera center and the landmarks are colinear, which is singular case (i). So we have $a>0$ in all non-singular cases.

Thus our quadratic is non-degenerate, and hence has at most two solutions
for $d_j$. Each solution for $d_j$ yields a 
unique solution for $d_i$ by \eqref{eq:di}, and then for $h_i$ and $h_j$ by \eqref{eq:height}, thereby 
showing that we have at most two solutions by Lemma~\ref{lem:hhdd}.

By Lemma~\ref{lem:hhdd}, we can characterize whether a solution of \eqref{eq:quad} yields a valid camera pose fitting with our landmarks 
by checking whether $d_i$ and $d_j$ are positive. We thus have two solutions to the pose problem precisely when there are 
two positive solutions for $d_j$ with the corresponding $d_i$ given by \eqref{eq:di} also being positive. 

Since there is at least one positive solution (as at least one pose---the ground truth---must satisfy our equations) 
and $a>0$, all the solutions of $ad_j^2+bd_j+c=0$ will be positive if and only if $c>0$, i.e.,
$$
	|\delta \tan\rho_i| > d.
$$	
And, still assuming there is a real solution, there will be two solutions provided that $b^2\ne 4ac$.

Suppose now that $\rho_j=\pi/2$.  Then $d_i=-\delta\tan\rho_i$ by \eqref{eq:di}. In this case, the camera is on the same 
horizontal plane $K_j$ as landmark $j$. Landmark $i$ must be above or below $K_j$ or else we have singular case (iii). If 
it is above $K_j$, then $\rho_i<\pi/2$ and $\delta<0$, so $d_i>0$. If landmark $i$ is below $K_j$, then 
$\rho_i>\pi/2$ and $\delta>0$, so again $d_i>0$. Thus, if $\rho_j=\pi/2$, we have two solutions
to the pose problem precisely when there are two solutions with $d_j>0$, i.e., when $|\delta\tan\rho_i|>0$ and $b^2\ne 4ac$.

Suppose that $\rho_j\ne\pi/2$. Repeating the proofs above with $i$ and $j$ swapped, we find that 
\begin{equation}\label{eq:quad2}
a^* d_i^2 + b^* d_i + c^* =0,
\end{equation}
where
$$
	a^* = 1-2\cos\beta\cot\rho_i\tan\rho_j + (\cot\rho_i\tan\rho_j)^2
$$
$$
	b^* = -2\delta\tan\rho_j (\cos\beta-\cot\rho_i\tan\rho_j)
$$
and
$$
	c^* = \delta^2\tan^2\rho_j - d^2.
$$
The solutions for $d_i$ here correspond one-to-one to the solutions $d_j$ to \eqref{eq:quad} according to \eqref{eq:di}, since 
neither $\rho_i$ nor $\rho_j$ is $\pi/2$. Thus assuming the pose problem has at least one solution, it has two solutions
precisely when there are two positive solutions for $d_j$ in \eqref{eq:quad} and all solutions are positive for $d_i$ in 
\eqref{eq:quad2}. Since \eqref{eq:quad2} has at least one positive solution (namely the one corresponding to the correct pose),
all solutions are positive provided $c^*>0$. Thus, we have two positive solutions if and only if both
\begin{equation}\label{eq:delta-rho}
	|\delta| \min(|\tan\rho_i|,|\tan\rho_j|) > d
\end{equation}
and
\begin{equation}\label{eq:neq}
	b^2 \ne 4ac.
\end{equation}
This is also true for $\rho_j=\pi/2$ under the stipulation that $|\tan (\pi/2)|=+\infty$, given what we have
shown earlier. 

Thus, \eqref{eq:delta-rho} and \eqref{eq:neq} characterize the cases where there are two solutions for pose in full generality,
assuming we are not in a singular case. Moreover, $s=\delta/d$ and $s_i=-\cot\rho_i$.
Thus \eqref{eq:delta-rho} says that 
$$
	\max(|s_1|,|s_2|) < |s|.
$$
If $0=|s|$ or $\max(|s_1|,|s_2|)=\infty$ then this cannot hold, and we always have a unique solution. We have thus shown that we have a unique solution if and only 
if either condition (a) of the theorem holds or else we have $0<|s|$, $\max(|s_1|,|s_2|)<\infty$ and $b^2=4ac$. 
To complete the proof we need to replace the condition $b^2=4ac$ with the condition that the the camera lies on the plane $L$.

Thus assume $0<|s|$ and $\max(|s_1|,|s_2|)<\infty$. We also have $|s|<\infty$ since we are not in singular case (ii).

Renumbering the landmarks, and rotating and translating the configuration, assume without loss of generality that 
the landmarks lie at $(-A,0,-B)$ and $(A,0,B)$, respectively, with $A$ and $B$ non-zero, and the camera position 
does not satisfy $c_3=-B$ so $\rho_1 \ne \pi/2$. Let $i=1$ and $j=2$.

We then have:
$$
	L = \{ (x_1,x_2,x_3) : Bx_1=Ax_3 \}
$$
for the plane in condition (b).  We have $d=2A$, $\delta=2B$, $d_1^2=((c_1+A)^2+c_2^2)^{1/2}$,
$d_2=((c_1-A)^2+c_2^2)^{1/2}$, $h_1=c_3+B$ and $h_2=c_3-B$. Furthermore $\cos\beta = (d_1^2+d_2^2-d^2)/(2 d_1 d_2)$ by the 
Law of Cosines, and $\cot\rho_1 = -h_1/d_1$ and $\cot\rho_2 = -h_2/d_2$. Plugging all this into Mathematica yields:
$$
	b^2-4ac = \frac{64 A^2 (Bc_1-Ac_3)^2}{(A^2-2Ac_1+c_1^2+c_2^2)(B+c_3)^2}.
$$
The denominator is zero only if $c_3=-B$, which we have assumed to be false, or $(c_1,c_2) = (A,0)$ which would 
imply that $|s_2|=\infty$. Thus $b^2-4ac$ is zero if and only if $\mathbold c \in L$. 
\end{proof}

It is likely that the discrepancy between Theorem~\ref{th:solutions} and the work of \cite{DGMP21} 
is due to two gaps in the latter. First, their work depends on deriving an equation of 
the form 
$$
	a\sin\alpha + b\cos\alpha + c=0
$$
for an unknown $\alpha$ determining the rotation matrix and for known quantities $a$, $b$ and $c$ (different from the ones in our proof).
This equation is correct, but in some cases it turns out that $a=b=c=0$, and at least one of the further
calculation depends on the expression $a/b$. Next the authors derive the equation
$$
	\alpha = \arccos(-c/(a^2+b^2))+\arctan(a/b)
$$
and claim that it has two solutions for $\alpha$, but they do not show both solutions correspond to a valid 
camera pose. (Compare how in our different geometric approach, sometimes an algebraic solution for $d_i$ or $d_j$ 
will be negative, in which case there is only one valid solution for pose, since distances cannot be negative.)

\section{Procedure}
We can summarize the procedure for obtaining the solution(s) from a camera image implicit in the proofs above, assuming
the landmarks are in a coordinate system with $z$-axis pointing upward.
\begin{enumerate}
\item Using the camera image and camera intrinsics (in case of pinhole camera), compute the unit image 
	vectors $\mathbold v_1$ and $\mathbold v_2$.
\item Use the accelerometer's gravity vector and $\mathbold v_1$ and $\mathbold v_2$ to compute the angles $\rho_1$, $\rho_2$ and $\beta$. 
\item Check to ensure we are not in a singular case. 
\item If $\rho_1$ or $\rho_2$ is in $\{0,\pi\}$, compute $d_1,d_2,h_1,h_2$ as indicated at the beginning of the proof of 
	Theorem~\ref{th:solutions}.
\item Otherwise, compute the solutions for $d_j$ (where $i$ is such that $d_i\ne\pi/2$ and $j=3-i$) using 
	\eqref{eq:quad}, and then compute $d_j,h_1,h_2$ using \eqref{eq:height} or \eqref{eq:height-2} and \eqref{eq:difference}. 
\item Discard any solutions with $d_1$ or $d_2$ negative.
\item Iterate over the remaining one or two solutions for $d_1,d_2,h_1,h_2$:
\begin{enumerate}
\item Find the intersection points of the circles of radii $d_1$ and $d_2$ around
	$(m_{11},m_{12})$ and $(m_{21},m_{22})$ respectively (at most two).
\item If $\beta\ne 0$, let $(c_1,c_2)$ be the intersection point such that the angle from $(m_{11},m_{12})$ to $(m_{21},m_{22})$ with vertex
	$(c_1,c_2)$ matches the sign of $\beta$.
\item If $\beta=0$, let $(c_1,c_2)$ be the unique intersection point (one of the $d_k$ is zero then).
\item Let $c_{3} = h_1 + m_{11}$. 
\item Let $T$ be a translation mapping $\mathbold c=(c_1,c_2,c_3)$ to $(0,0,0)$.
\item Let $\mathbold {v'}_i = (T\mathbold m_i)/|T\mathbold m_i|$.
\item Compute the rotation $R$ that transforms $\mathbold{v'}_1$ and $\mathbold{v'}_2$ to the image vectors $\mathbold v_1$ and 
	$\mathbold v_2$ respectively.
\item The transformation from the object coordinate system to the camera coordinate system is given by $RT$. 
\end{enumerate}
\end{enumerate}

\section{Unlabeled landmarks}
Let us now consider a variant situation where the camera is unable to distinguish the two landmarks from each other.
It follows immediately from Theorem~\ref{th:solutions} that, outside of the singular cases (i)--(iii), there are at most four
solutions. But it turns out that we can reduce that maximum number to two.

\begin{lem}\label{lem:identify}
If in the setting of Theorem~\ref{th:solutions} the ground truth camera pose is such that there are exactly two 
solutions for camera pose with landmarks labeled, and $\rho_1,\rho_2,\beta$ are the image angles 
to $\mathbold m_1$ and $\mathbold m_2$, then there is no camera position where $\rho_1,\rho_2,\beta$ are the 
image angles to $\mathbold m_2$ and $\mathbold m_1$.
\end{lem}

In other words, in this case, if we make a guess of which sensor pixel corresponds to which landmark, we can
check the correctness of the guess, and hence we can use the sensor image to label the two landmarks.
It follows that the only non-singular case where we might not be able to label the two landmarks is when 
Theorem~\ref{th:solutions} yields a unique solution. In that case, neither possible guess as to the identifications
of the two landmarks can lead to two solutions by the Lemma, and hence there are again at most two solutions,
one for each identification guess. (An example is where the two landmarks are at the same height and the camera 
is at a different height: With landmarks unlabeled there will be two solutions for pose related by a rotation
by angle $\pi$ around the vertical line through the midpoint between the two landmarks.) Therefore we have:

\begin{thm}\label{th:isolutions}
	In the setting of Theorem~\ref{th:solutions}, if we are not in a singular case and the landmarks are not
	labeled, there are at most two solutions for camera pose.
\end{thm}

It is worth noting that there are cases with a unique solution. For instance, suppose the landmarks are at 
different heights, and the camera is located within the unique solution region of Theorem~\ref{th:solutions} 
\textit{and} both above the lower landmark and below the upper one (it's clear from Figure~\ref{fig:cones} that this 
region is non-empty). Then the tilt angles will allow the landmarks to be identified and labeled: only the tilt angle to the higher 
landmark will be less than $\pi/2$.

\begin{proof}[Proof of Lemma~\ref{lem:identify}]
Suppose we are in a non-singular case and there are two solutions for camera pose. By Theorem~\ref{th:solutions},
both tilt angles $\rho_i$ are in $(0,\pi)$. If either image tilt angle
$\rho_i$ is $\pi/2$, the other image tilt angle must not be $\pi/2$ or we would have a singular case. The two 
landmarks are then at different heights, and we can tell which landmark the tilt angle that's not $\pi/2$ corresponds to, 
since if that tilt angle is less than $\pi/2$, it corresponds to the higher landmark. 

So suppose neither tilt angle is $\pi/2$. Let $i=1$ and $j=2$. In the notation of Theorem~\ref{th:solutions},
the quadratic \eqref{eq:quad2} must two positive solutions (they cannot be zero, since that would not fit a tilt
angle in $(0,\pi)$). Since $a^*>0$ by the same reasoning by which we 
showed that $a>0$ in the proof of the theorem, this requires the first derivative of the quadratic at the origin to 
be negative, so $b^*<0$, and the value at zero to be positive, so $c^*>0$. Now note that if we let $a',b',c'$ be the 
coefficients in \eqref{eq:quad} calculated using the 
image tilt angles $(\rho_1',\rho_2',\beta')=(\rho_2,\rho_1,-\beta)$, then $b'=-b^*$ and $c'=c^*$. Hence, $b'>0$ and $c'>0$. 
Thus the quadratic in \eqref{eq:quad} for the primed image angles has a positive first derivative, a positive coefficient of 
the second degree term, and is positive at zero, so it cannot have any positive solutions. Thus, there are no solutions 
to the labeled landmark pose problem corresponding to the primed image angles. 

But finding a solution to the pose problem for labeled landmarks $\mathbold m_1$ and $\mathbold m_2$ and the 
image angles $\rho_1',\rho_2,\beta'$ is equivalent to finding a solution to the pose problem for labeled landmarks
$\mathbold m_2$ and $\mathbold m_1$ and image angles $\rho_1,\rho_2,\beta$. 
\end{proof}

\section{Horizontally coplanar landmarks}\label{sec:coplanarity}
As noted, one way to ensure the uniqueness of the solution is to place the landmarks on the same 
horizontal plane and constrain the camera to always lie above or below them. For instance, in some applications
the landmarks could be placed on a level floor, with the camera guaranteed to be above them. 

Additionally in this case, the calculations simplify which could aid real-time implementation on inexpensive hardware. 
In the notation of  the proof of Theorem~\ref{th:solutions}, we have $\delta=0$, and $\rho_1$ and $\rho_2$ are in either 
of the intervals $(0,\pi/2)$ or $(0,\pi/2)$. Let $i=2$
and $j=1$. Then in \eqref{eq:quad} we have $b=0$ and $c=-d^2$, so for the positive solution we have:
$$
	d_1 = \frac{d}{(1-2\cos\beta\cot\rho_j\tan\rho_i + (\cot\rho_j\tan\rho_i)^2)^{1/2}}
$$
and
$$
	d_2 = d_1\cot\rho_1\tan\rho_2
$$
while 
$$
	h_1 = h_2 = -d_1 \cot\rho_1.
$$

A consumer product with two landmarks on a horizontal line and a camera containing an accelerometer has existed for a while.
The Nintendo Wii game console was equipped with a remote gaming controller containing an accelerometer and an infrared 
camera with firmware that tracks up to four infrared landmarks, as well as an emitter bar containing two infrared
landmarks (each consisting of a cluster of light emitting diodes). In typical use, the emitter bar was centrally horizontally 
located below or above a television screen, and the Wii remote was used to operate an on-screen pointer. If the position of the emitter bar relative to the 
screen and the size of the screen were perfectly known, Theorem~\ref{th:solutions} would in principle allow the computation of the position 
of an on-screen pointer directly pointed at by the central axis of the controller, as long as the controller was not in the same plane as the 
emitter bar, and the acceleration due to user movement was filtered out. It is, however, worth noting that historically the Wii's operating system did 
\textit{not} attempt to precisely position the on-screen pointer along the axis of the controller. 

Preliminary testing showed that the precision of the Wii remote with two same-altitude infrared landmarks placed 
under a television screen and accelerometer (P2PA+accelerometer or P2PA) yields unsatisfactory pointing 
precision. The on-screen pointer moved jerkily and the positioning was not very accurate. This was compared against a control system based 
on using the Wii remote with P4P (and P3P plus heuristics as a fallback in case a landmark was invisible) with 
two additional infrared landmarks placed above the screen. The P4P approach provided entirely satisfactory pointing 
precision, and was sufficient to satisfactorily play an emulated version of Nintendo Duck Hunt, while the P2PA 
approach made the game frustrating. It is likely that using P2PA for pointing presents a special implementation 
challenge in games, due to the large amount of user movement which makes extracting the low-frequency gravity 
component of the acceleration more difficult.

As we saw, when the landmarks \textit{and} camera lie on the same horizontal plane, we have a singular 
case with infinitely many solutions. It is worth noting that if one adds a third landmark anywhere else on the same 
horizontal plane gains one a unique solution everywhere except on the line or circle through the three landmarks (depending
on whether they are colinear or not). For any one of the tilt angles is sufficient to show the camera must lie on 
the horizontal plane through the landmarks. The horizontal image angle between any two landmarks then constrains the 
camera to lie on a specific circular arc or line through those landmarks. Assume the camera does not lie on a circle 
or line through all three landmarks. Then if the landmarks are $A$, $B$ and $C$, the circle
or line through $A$, $B$ and the camera and the circle or line through $B$, $C$ and the camera will meet in at most two 
points. One such meeting point is $B$, but we have assumed the camera is constrained not to be at a landmark, so the
position of the camera is uniquely determined. 

Going back to the simple P2PA case of two landmarks on a horizontal line, it is of some practical interest to investigate what 
happens to the precision of pose determination when the camera altitude approaches the altitude of the two landmarks, given 
the finite resolution of the camera sensor images and of the accelerometer. We thus simulate an arrangement with 
randomized errors. The simulation involves the landmarks $150$~mm apart, and the virtual camera placed $500$~mm away from 
the line through the landmarks. In one experiment, the camera was in front of the left landmark, and in the other
it was mid-way between the landmarks. The camera was then moved vertically over an altitude range from $0.001$~mm to $500$~mm,
measured relative to the landmarks. 

Two sources of error were included. First, considering the camera to be spherical and thus sensing the direction of 
each landmark as a unit vector $\mathbold v_i$, a random error was added resulting in the sensed direction being 
uniformly distributed over a circular region of angular radius $0.36^\circ$ centered on the ground truth. Additionally, 
a random error uniformly distributed over $[-0.001g,0.001g]$ was added to the ground truth of each accelerometer axis. 
For each experiment, $1000$ vertical positions were tested, with $1000$ samples at each position used to calculate
the rms error using the two landmarks and accelerometer data. See Figure~\ref{fig:sim} for results.

\begin{figure}
\centering
  \begin{minipage}[b]{0.45\textwidth}
    \includegraphics[width=\textwidth]{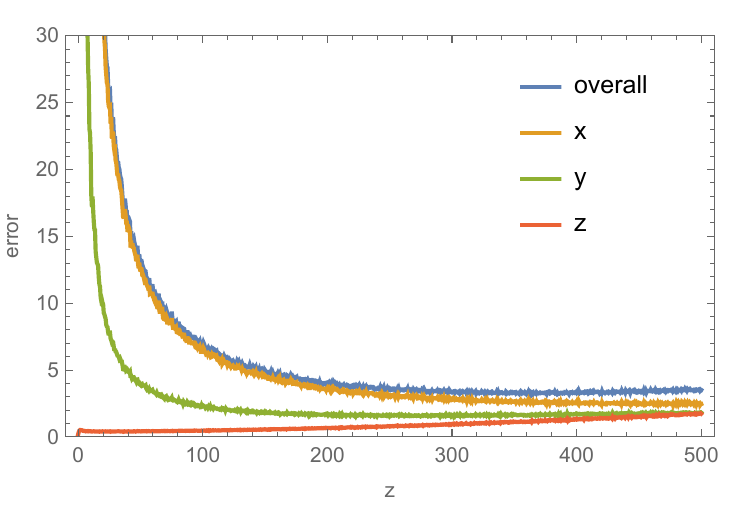}
  \end{minipage}
  \hfill
  \begin{minipage}[b]{0.45\textwidth}
    \includegraphics[width=\textwidth]{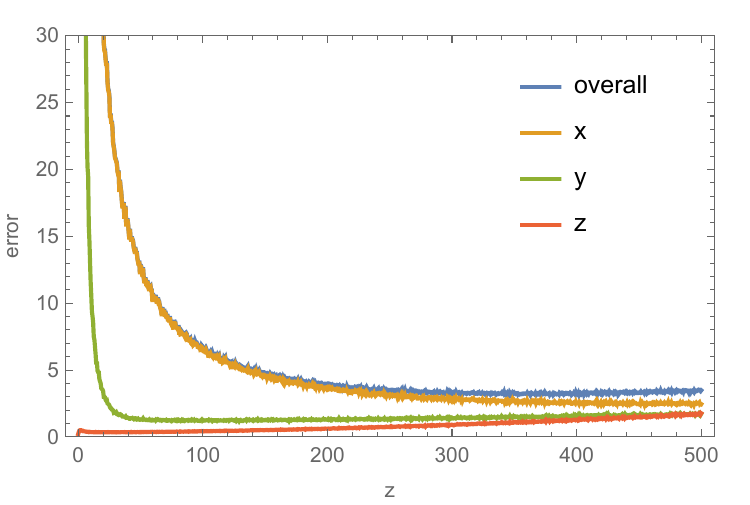}
  \end{minipage}
\caption{Simulated error with camera aligned with left landmark (left) and midpoint between landmarks (right).}
\label{fig:sim}
\end{figure}
As expected, overall error was very high near the singularity that occurs when the camera is in the same plane as the 
landmarks. At 1~mm altitude, the overall rms error was 354~mm for the camera in front of the left landmark and 340~mm in front 
of the center landmark. The rms error dipped below 20~mm at altitudes of approximately 33~mm and 32~mm, respectively. 

With a coordinate system where the landmarks are at $(0,0,0)$ and $(150,0,0)$, so the $x$-axis goes through the landmarks,
and the camera is located at $(0,-500,z)$ or $(75,-500,z)$ depending on the experiment, we also observed the positioning
error in each axis separately. Altitude error ($z$) was small even near the singularity at $z=0$, which is to be expected since
it is possible to detect the camera altitude correctly even at the singularity. The errors in $x$ and $y$ were both very large near
the singularity, but the $x$-error accounted for the bulk of the overall error at low to moderate camera altitudes. 

The practical consequence is that if one uses the configuration with landmarks on the same horizontal plane to ensure 
unique solution, one should constrain the camera to be sufficiently far from this plane to ensure acceptable accuracy.

\section{Physical implementation on smartphone}
Smartphones typically include both a camera and an accelerometer, and hence should be capable of identifying pose 
using the methods described here. An algorithm for the case where the landmarks lie on the same horizontal plane 
was implemented in Java for Android phones using the OpenCV image processing library and tested with a Google Pixel~7
Pro phone's main camera, with camera intrinsic matrix computed with OpenCV's calibration algorithm and a calibration
image.~\cite{Pruss24} Four ArUco fiducial markers~\cite{Aruco} were printed on a sheet letter size paper. Each marker was a square 
$50$~mm per side, with the marker centers describing a rectangle with sides $150$~mm and $200$~mm. The centers of two 
markers spaced $150$~mm apart were used as landmarks $1$ and $2$. The remaining two markers were used as controls.
The software was designed to work for two cases, in both of which the landmarks lie on the same horizontal plane.
In one case, the image was placed horizontally with both landmarks and controls on the same plane, and in the other it 
was placed vertically with controls above the landmarks.

\begin{figure}
\includegraphics[height=8cm]{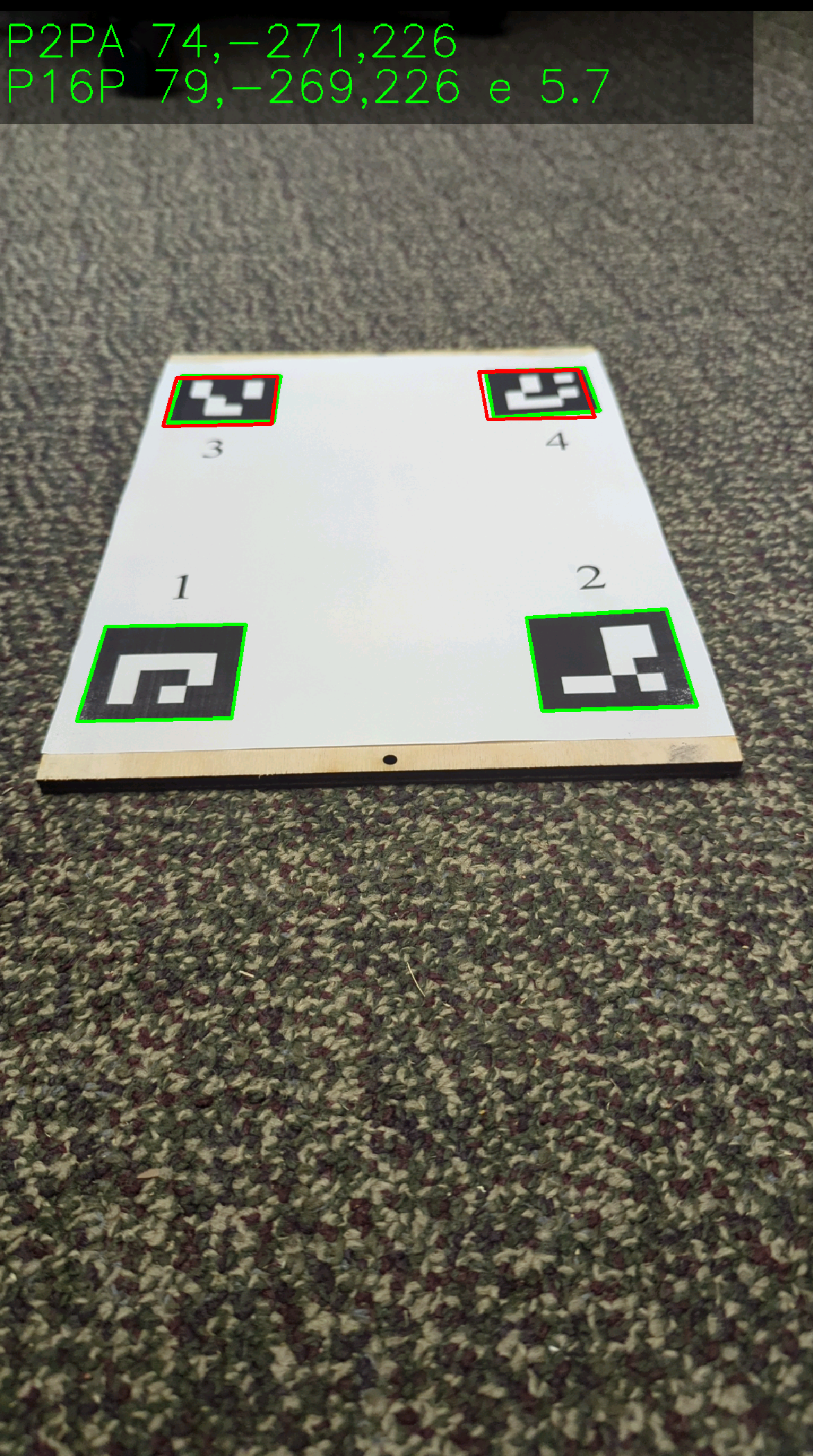}
\caption{Markers 1 and 2 are landmarks, and markers 3 and 4 are controls. Red outlines of controls are computed with P2PA.}
\label{fig:phone}
\end{figure}
The software outlines all four markers in green, and computes the camera position and rotation using the two landmarks and 
accelerometer and the methods in the proofs in this paper (P2PA). Once pose is computed, the physical position of the control 
markers can be used to computed expected values for the position of the images of the control markers in the camera image,
and these are outlined in red, providing a quick visual confirmation of the correctness of the camera pose (see Figure~\ref{fig:phone}).
Visually, typically the red outlines are sufficiently close to the correct positions that one could use this method as 
a sufficiently accurate control for an on-screen pointer or a game involving targeting.

Additionally, the corners of all four markers were used with OpenCV's \texttt{Calib3d.solvePNP()} method to 
compute camera pose, and the P2PA computation was tested against this P16P calculation. With the 
marker page horizontally on the floor, for a wide range
of phone positions at approximately $500$~mm above the floor, the P2PA calculation was observed to lie within $5$ to 
$15$~mm of the P16P calculation when the phone was held in place in one hand. During deliberate movements of the phone, the error would go up to 
around $35$~mm, presumably due to the detection of the gravity vector being affected by acceleration (note that the software 
used the default Android gravity filtering). This suggests that P2PA may have more limited applicability in gaming or on an 
accelerating drone. 

\begin{figure}
\centering
  \begin{minipage}[b]{0.45\textwidth}
    \includegraphics[width=\textwidth]{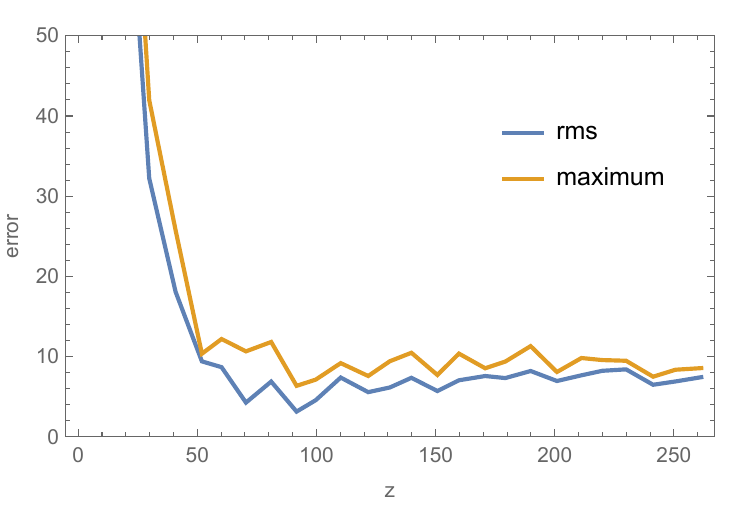}
  \end{minipage}
  \hfill
  \begin{minipage}[b]{0.45\textwidth}
    \includegraphics[width=\textwidth]{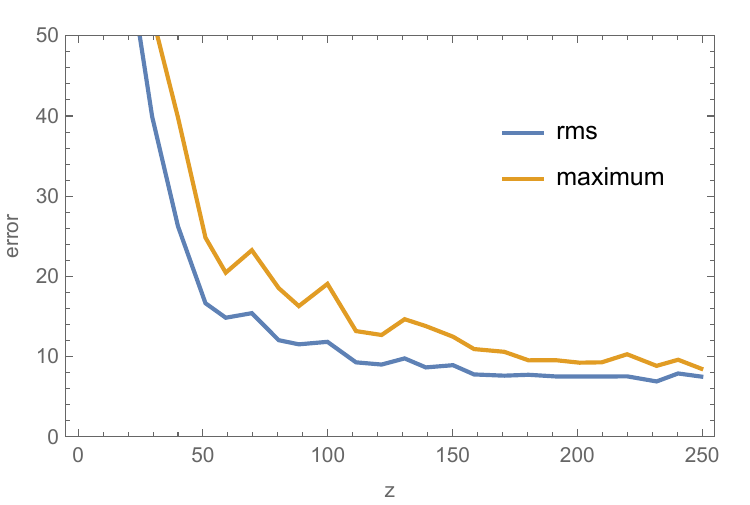}
  \end{minipage}
\caption{Error with camera aligned with left landmark (left) and midpoint between landmarks (right).}
\label{fig:exp}
\end{figure}

For a more systematic test, the phone was attached to a tripod, with camera tilted approximately 
$6^\circ$ down and placed approximately orthogonally to the wall in azimuth. The marker page was posted vertically on a wall. 
The tripod allowed the camera to be moved up and down. Camera positioning errors for our P2PA algorithm were estimated at 
different camera heights relative to the landmarks, with P16P results standing in for the ground truth. In one experiment, the camera 
was placed approximately 
horizontally mid-way between the markers (using P16P data for positioning) and in another it was approximately horizontally 
aligned with the centers of the left markers. The $z$-axis always points upward, starting at $0$ with the landmarks. The phone was
moved upwards in increments of approximately $10$~mm (measured by P16P), and at each elevation 50 data points (P16P and 
P2PA) were captured. 

The markers were laser-printed and attached to a laser-cut plywood rectangle with a symmetrically oriented hole at the top 
for a nail. Leveling was achieving simply by hanging on a wall and letting it stabilize under gravity. The experiment 
also relied on the standard factory calibration of the phone 
accelerometer. 

The average and maximum P2PA errors for each data collection are plotted in Figure~\ref{fig:exp}. 
As we approach the singularity at $z=0$, altitudes below approximately 30~mm are out of range on the graphs, and involve 
unusably high overall errors, namely an average of 65~mm at approximately 20~mm altitude, and 242~mm at approximately 10~mm 
altitude (both averaged between both experiments).

Nonetheless, the P2PA estimate of the $z$-coordinate continued to be usable: the maximum error in $z$ observed in either
experiment was $6$~mm. This corresponds to the geometrical fact that the camera can be correctly identified to lie at the horizontal
plane of the landmarks by the fact that $\rho_1=\rho_2=0$. 

\section{Conclusions}
A precise characterization of the multiplicity of solutions for the P2P problem with accelerometer and gravity-aligned 
landmarks is given, expanding previous work of \cite{KBP10} and correcting previous work of \cite{DGMP13,DGMP14,DGMP21}. Of 
particular interest is the case of landmarks at the same 
altitude, as that yields unique solutions except on the horizontal plane passing through the landmarks. The resulting 
pose-recovery technique can be implemented on consumer-level phone hardware, though it may not be sufficient for pointer-
and aiming-based games, or for location of a quickly accelerating drone, due to error in extracting the gravity vector
from accelerometer data. It is also noted that when the landmarks are unlabeled, there are still always at most two solutions,
and there are still some cases with one solution.


\begin{thebibliography}{10}
\bibitem{WHZ} Wang, B., Hu, H., Zhang, C.: New insights on multi-solution distribution of the P3P problem. Imag.\ Vis.\ Comput.
	103:104009 (2020)

\bibitem{KBP10} Kukelova, Z., Bujnak, M., Pajdla, T.: Closed-form solutions to the minimal absolute pose problems with 
known vertical direction. Proc.\ ACCV, pp.~216–229 (2010)

\bibitem{DGMP13}
D'Alfonso, L., Garone, E., Muraca, P., Pugliese, P.: P3P and P2P Problems with known camera and object vertical 
	directions. In: Proc.\ 21st\ IEEE\ MED, pp.~444–451 (2013)

\bibitem{DGMP14}
D'Alfonso, L., Garone, E., Muraca, P., Pugliese, P.: On the use of IMUs in the P$n$P problem. In: Proc.\ IEEE ICRA, pp.~914–919 
(2014)

\bibitem{DGMP21} 
D'Alfonso, L., Garone, E., Muraca, P., Pugliese, P.: Camera and inertial sensor fusion for the PnP problem: 
	algorithms and experimental results. Mach.\ Vis.\ Appl.\ 32:90 (2021)
	
\bibitem{Pruss24} Pruss, A.: P2P experiment source code. \url{https://github.com/arpruss/p2pexperiment} (2024)

\bibitem{Aruco}
Garrido-Jurado, S., Mu\~{n}oz-Salinas, R., Madrid-Cuevas, F., Mar\'\i{}n-Jim\'{e}nez. 
	Automatic generation and detection of highly reliable fiducial markers under occlusion.
	Pattern Recognit.\ 47:2280--2292 (2014)


\end{thebibliography}
\end{document}